\documentclass[a4paper]{styles/svproc}

\usepackage{url}

\usepackage{caption}
\usepackage{tabularx}
\usepackage[export]{adjustbox}
\usepackage{subcaption}
\usepackage{graphicx}
\usepackage{amsmath}
\usepackage[linesnumbered,algoruled,boxed,vlined, noend]{algorithm2e}
\usepackage{comment}
\usepackage{amsfonts}
\usepackage{mathtools}
\usepackage{overpic}
\usepackage{xcolor}
\usepackage[pdftex]{pict2e}
\usepackage{wrapfig}
\usepackage{multicol}
\usepackage{lipsum}
\usepackage{float}

\newenvironment{myitem}{\begin{list}{$\bullet$}
{\setlength{\itemsep}{0pt}
\setlength{\topsep}{0pt}
\setlength{\leftmargin}{12pt}
\setlength{\parsep}{0pt}
\setlength{\partopsep}{0pt}}}%
{\end{list}}

\usepackage{soul}

\begin{document}
\mainmatter              
\title{\large Safe, Occlusion-Aware Manipulation for\\
Online Object Reconstruction in Confined Spaces}
\titlerunning{Safe Occlusion-Aware Manipulation}  
%
\author{Yinglong Miao \and Rui Wang  \and Kostas Bekris}%
\authorrunning{Yinglong Miao et al.} 
\tocauthor{Yinglong Miao, Rui Wang, and Kostas Bekris}%
\institute{Rutgers University, New Brunswkick, NJ, 08901, USA}
\maketitle              

\vspace{-.15in}
\begin{abstract}
Recent work in robotic manipulation focuses on object retrieval in cluttered spaces under occlusion. Nevertheless, the majority of efforts lack an analysis of conditions for the completeness of the approaches or the methods apply only when objects can be removed from the workspace. This work formulates the general, occlusion-aware manipulation task, and focuses on safe object reconstruction in a confined space with in-place rearrangement. It proposes a framework that ensures safety with completeness guarantees. Furthermore, an algorithm, which is an instantiation of this abstract framework for monotone instances is developed and evaluated empirically by comparing against a random and a greedy baseline on randomly generated experiments in simulation. Even for cluttered scenes with realistic objects, the proposed algorithm significantly outperforms the baselines and maintains a high success rate across experimental conditions.
\end{abstract}
\section{Introduction}
Robotic manipulation has wide applications in manufacturing, logistics, and domestic domains. Unstructured scenes, such as logistics and home environments, require perception (mostly vision) for scene understanding \cite{bohg2017interactive}. Thus, uncertainty in perception, such as occlusion, complicates the manipulation process. This work aims to address this issue; specifically, to safely reconstruct unknown objects online in a confined space and fully reveal the scene through in-hand sensing and in-place relocation, given a static RGB-D camera.

Such online object reconstruction is valuable functionality for manipulation in realistic scenarios where novel and unknown objects can appear frequently and challenge model-based methods \cite{lu2021online}. Most applications, such as object retrieval in clutter and object rearrangement (e.g., setting up a dinner table), may require reconstructing only a subset of the objects in the scene, however, it is essential to understand how to reconstruct objects in these unstructured scenes with occlusion safely. Hence this work focuses on the case where all objects need to be reconstructed safely to completely reveal the scene while avoiding potential collisions. This  approach contributes a foundational understanding of the object reconstruction domain and the conditions under which safety and  probabilistic completeness can be achieved.

Addressing this task, however, can be challenging for a variety of reasons. For instance, robot reachability is restricted in confined spaces, such as shelves; this complicates viable object placements and planning safe robot movements. Reachability constraints introduce dependencies between objects, which restrict the order with which objects can be manipulated safely \cite{wang2022lazy}. On the other hand, occlusion among objects occurs given a single-camera view of the confined space. In such cases, it is hard to extract occluded objects safely, and thus poses another constraint on the order of object movements.
These two challenges are interleaved and need to be solved simultaneously.

The majority of work in robot manipulation has assumed known object models \cite{dogar2014object,nam2021fast,xiao2019online,li2016act}. More recently, the focus has shifted to cases where object models are unavailable. Nevertheless, most of these methods do not provide theoretical guarantees, either due to the dependence on machine learning 
\cite{bejjani2021occlusion,danielczuk2019mechanical,huang2020mechanical,huang2022mechanical}, or assumptions that are valid only in limited scenarios \cite{gupta2013interactive}. Learning-based methods bring great promise to achieve efficiency and are highly 
desirable if they can be applied within a safe and complete framework. This work aims to provide such a framework on top of which more efficient solutions can be built.

In summary, this work first categorizes problems in object manipulation under occlusion and surveys past work in this field. Then, the task of safe object reconstruction in a confined space with in-place relocation is formalized. This formalization geometrically defines occlusion and the incurred object relationships. This leads to a probabilistically complete algorithmic framework with accompanying proof. The framework is general and can have different instantiations based on the implementation of primitives. An instantiation for monotone problems is proposed, which empirically shows significant improvements against a random and a greedy baseline in terms of computation time and number of actions taken, with a high success rate
\footnote{Code and videos: \url{sites.google.com/scarletmail.rutgers.edu/occlusion-manipulation}.}.

\section{Parameters of Occlusion-Aware Manipulation}
This section considers parameters that define the difficulty of occlusion-aware manipulation problems.

\textbf{Uncertainty in perception}
Continuous uncertainty arises mainly due to noise in perception, while  discrete uncertainty arises due to occlusions. Most efforts have focused on continuous uncertainty by introducing a probability distribution (e.g., a Gaussian PDF) to represent a belief state. In contrast, \cite{wang2020safe} explicitly models the uncertainty as a discrete set of hypotheses and constructs a motion planner for a picking task. Discrete uncertainty can be modeled via geometric reasoning. This work focuses on this direction and assumes perfect sensing without continuous noise. Alternatives in the literature do not explicitly model visual uncertainty and often apply learning-based methods \cite{qureshi2021nerp}.

\textbf{Workspace Constraints: Confined spaces}
Confined spaces, such as shelves, which are the focus of this work, are more constraining than a tabletop, which is a setup widely studied \cite{xiao2019online,price2019inferring,novkovic2020object}, since lateral access limits grasp poses. 

{\bf External Buffers} Problem difficulty also depends on whether external buffers for object placement exist. When unlimited external buffers are provided, objects can be gradually extracted. Otherwise in-place rearrangement is necessary, which requires finding placement poses \cite{gao2021fast,cheong2020relocate,ahn2021integrated}.

\textbf{Object Model Knowledge} The level of knowledge about objects can range from known models, to category-level shape knowledge to unknown object models. With known object models, objects are often assumed to be perfectly recognized once a part becomes visible \cite{dogar2014object}; or with an error that is modeled probabilistically \cite{li2016act}. Category-level knowledge can provide hints to estimate object models. For instance, when objects are cylinders, the pose and size can be estimated from visible parts \cite{nam2021fast}. When object models are unknown, objects may need to be reconstructed \cite{lu2021online,gupta2013interactive}. When the number of hidden objects is given, their poses can be sampled inside the occlusion region \cite{bejjani2021occlusion,xiao2019online,li2016act}. Otherwise, the scene needs to be fully revealed when it is unknown in the worst case.

\textbf{Clutter and Object Relationships}  The level of clutter affects the ability to grasp objects and place objects freely \cite {cheong2020relocate,ahn2021integrated}.   Furthermore, objects may interact in non-trivial ways. For instance, objects can be stacked on top of each other, which constrains the sequence of grasps \cite{kumar2021graph,qureshi2021nerp}. Alternatively, objects can be physically unstable in unstructured scenes, such as heaps and piles \cite{danielczuk2019mechanical,kurenkov2020visuomotor}. 

\textbf{Camera Placement} The camera, when attached to the robot arm, can be moved to reconstruct the scene \cite{bejjani2021occlusion}. On the other hand, when the camera is static, the robot must manipulate objects to place them in the camera's field of view, which requires solving motion planning problems under visibility constraints.

\textbf{Action Type} A mobile-based robot has higher reachability than a fixed base one. Non-prehensile actions allow more flexibility, such as moving groups of objects, than prehensile actions but increase the level of uncertainty. Manipulation Among Movable Obstacles \cite{stilman2007manipulation} often employs such non-prehensile operations.

\textbf{Combinatorial Aspects}
A monotone rearrangement problem is one where objects need to be moved at most once. Non-monotone problems are harder instances of rearrangement \cite{wang2021efficient}. This work aims to extend notions of monotonicity in the context of manipulation under occlusion and visibility constraints.

\section{Related Work}
Object manipulation are generally solved by Task and Motion Planning (TAMP) \cite{garrett2021integrated}, where past works have focused on object rearrangement in cluttered and confined spaces \cite{wang2021efficient,wang2022lazy}. Interactive Perception focuses on the interplay of perception with manipulation \cite{bohg2017interactive,garrett2020online}. 

\textbf{Model-Based Object Retrieval in a Shelf}
Many works in the literature focus on object search and retrieval from a shelf with known object models or through learning methods
\cite{dogar2014object,nam2021fast,bejjani2021occlusion,xiao2019online,li2016act,huang2020mechanical,huang2022mechanical}. In this setting, the perception module is assumed to fully recognize the object once a sufficient part is observed, and ignores partial occlusions. Prior work uses a graph to encode object occlusion and reachability dependencies and constructs an algorithm to solve the problem optimally \cite{dogar2014object}.
An alternative encodes the traversability of objects through a graph named T-graph, which limits to cylinders \cite{nam2021fast}. Both of these efforts remove objects from the scene, in which case completeness is trivial to satisfy.


\textbf{POMDP-based approaches} The POMDP framework has been used to formulate the problem as "Mechanical Search" \cite{danielczuk2019mechanical}, which has been applied in different settings \cite{huang2020mechanical,huang2022mechanical}. Various works have relied on POMDP formulations \cite{xiao2019online,bejjani2021occlusion,li2016act}, which are solved by standard solvers or Reinforcement Learning (RL) algorithms. In this context, the trustworthiness of object recognition has been explicitly modeled based on the occlusion ratio \cite{li2016act}. Learning-based methods have been proposed to sample target poses \cite{huang2020mechanical}, which are then used to guide a tree-search framework with an RL policy \cite{bejjani2021occlusion}. Some of the efforts use learning-based estimators to predict the target object pose in 1D along the image width, which then guides the planner to find the policies, but is incomplete in solving the problem compared with baselines that have full knowledge of the scene \cite{huang2020mechanical}\cite{huang2022mechanical}.
The most related effort to this work \cite{gupta2013interactive} constructs a multi-step look-ahead algorithm that completely explores occluded regions in a grid-world scenario. While it extends the algorithm to realistic settings, the proposed perception module is ad-hoc.
Meanwhile, the planner relies on a sampling of actions to move valid objects instead of analyzing the dependency structure, which may provide better properties for more efficient planning.
Furthermore, the method does not reconstruct objects to retain information and better represent the objects.

\textbf{This Work}
A complete solution can be easily constructed with known object models by combining prior work. This work considers a complicated task with unknown object models and focuses on safe object reconstruction with in-place relocation.
At the same time, this work introduces a set of assumptions for theoretical analysis. In particular, objects are assumed to be convex, specifically cylinders or rectangular prisms, and do not stack on top of each other. 
This work adopts a TSDF representation \cite{newcombe2011kinectfusion}, similar to prior efforts \cite{lu2021online}, to reconstruct object models. To ensure safety, it models the object by a conservative volume, and an optimistic volume following prior work \cite{mitash2020task}. In-place object rearrangement is a hard problem since the simplified version of circle packing in 2D is already NP-hard \cite{demaine2010circle}. Recent works \cite{cheong2020relocate,ahn2021integrated} sample placements and evaluate them based on reachability conditions, but limit to cylindrical objects. Instead, this work applies to general-shaped objects by a novel convolution-based method to sample collision-free placements, similar to the method used in \cite{huang2022mechanical}.

%
%

\section{Problem Formulation}
A workspace $W$ is defined as the volume inside a shelf with a supporting surface where objects can be placed stably. At any moment, it is composed of free, collision, and occlusion space, given the robot's knowledge. An object model is defined as the set of points occupied by the corresponding object given an object pose. The ground-truth object model is denoted as $X_i(s_i)$, which indicates the set of points occupied by object $o_i$ at pose $s_i$. As only an estimate of the object model is available during execution, two approximations are used:
\begin{myitem}
\item a conservative volume $X^{con}_i(s_i)$ that contains the space occluded or occupied by object $i$; 
\item an optimistic volume $X^{opt}_i(s_i)$ that contains the revealed surfaces of object $i$.
\end{myitem} 

At any time, these subsets satisfy that $X^{opt}_i(s_i)\subseteq X_i(s_i)\subseteq X^{con}_i(s_i)$. The conservative volume is used during motion planning for collision checking to ensure safety (discussed later in this paper); the optimistic volume is used to compute a picking pose for a suction gripper. $s^0(i)$ denotes the initial pose of object $o_i$. A robot is represented by its geometry $R(q)$, which is the set of points occupied by the robot at the robot state $q$ (e.g., the set of the robot's joints' angles). A suction cup is attached to the robot arm and can perform two actions: attach and detach. After attaching, the object is assumed to have a fixed transformation relative to the gripper until detached.
The prehensile action of picking and placing is parameterized as $(o_i,s_i,\tau_i)$, which encodes moving object $i$ to pose $s_i$ via the robot trajectory $\tau_i$.


\textbf{Occlusion}
A pinhole camera is fixed relative to the robot base, with a pose denoted as $c$. Perception is performed at discrete-time steps after each manipulation action. The occlusion space due to object $o_i$ at time $t$ is then represented as: 
$O_i(s_i,t)\equiv
\{c+\alpha (x-c):\alpha\in(1,\infty),x\in X_i(s_i)\},$
which arises from the RGBD image captured by the camera $c$ at time $t$. 
The total occlusion from all objects at time $t$  is defined as
$O(t)=\bigcup_i O_i(s_i,t)$. The variable $t$ will be often omitted if its value is obvious or unimportant, resulting in the occlusion space $O_i(s_i)$ for an object pose $s_i$ of object $o_i$. The occlusion space due to object $o_i$ up to time $t$ given the observation history is then $O_i(s_i,1..t)\equiv O_i(s_i,t)\cap\bigcap_{k=1}^{t-1}O(k)$.
The total occlusion space up to time $t$ is $\bigcap_{k=1}^tO(k)$.

The visible surface of an object $o_i$ at pose $s_i$ can be extracted from the image:
$F_i(s_i)=
\{x:\phi(x)\in\phi(X_i(s_i)),d(x)=depth(\phi(x))\}$.
Here $\phi(\cdot)$ denotes the projection of the 3D geometry on the 2D image; $d(x)$ denotes the depth of point $x$ relative to the camera; and $depth(\phi(x))$ denotes the value of the depth image at the projected pixel $\phi(x)$. 
Due to occlusion, an object may not be fully visible given an image. An object is defined to be \textbf{fully revealed} at time $t$ if all points on its visible surface have been revealed up to time $t$. It is \textbf{partially revealed} if some points on the visible surface of its model have been seen but not all yet. Otherwise, the object is \textbf{hidden}, and the camera has not seen any part of the visible surface of its model. 

\textbf{Safety}
The notion of safety considered here imposes two requirements. First, the conservative model of an object should not collide with other objects or the scene. Meanwhile, an object should not collide with hidden objects or parts, which requires a careful analysis of the extraction motion\footnote{Please refer to the Appendix for an analysis of object extraction. Appendix can be found at:
\url{https://arxiv.org/abs/2205.11719}.}.
Extraction motions can be found for specific object shapes, such as cylinders and rectangular prisms, which are assumed by this work. For such cases, this work uses an extraction direction perpendicular to the shelf opening in its reasoning to guarantee safety.

Thus, the problem is to find a sequence of actions
$[(o_i,s_i,\tau_i)]$ that allows to fully reconstruct all objects (i.e., remove the occlusion volume) safely.


\section{Algorithms and Analysis}
\SetKw{Continue}{continue}
\SetKw{Break}{break}

\newcommand{\bd}{bd}
This section defines a data structure to represent object dependencies and then illustrates a theoretical framework that attains probabilistic completeness. A concrete instantiation is introduced later to solve monotone tasks.

\subsection{Occlusion Dependency Graph (ODG)}
Our algorithm depends on a graph which represents visibility constraints among objects.
To construct the graph, the occlusion relationship between objects needs to be defined.
Recall that the occlusion space of an object $i$ is defined to be $O_i(s_i)$.
Then for objects $i$ and $j$ at poses $s_i$ and $s_j$, the part of $j$ occluded by $i$ can be defined as $O_{ij}(s_i,s_j)\equiv O_i(s_i)\cap X_j(s_j)$.
Following this, object $i$ at pose $s_i$ is defined to \textbf{occlude} (directly or indirectly) another object $j$ at pose $s_j$ if $O_{ij}(s_i,s_j)\neq\emptyset$.
As objects may jointly occlude a region, it is helpful to separate the occlusion contribution of each object. To achieve this,
the \textbf{direct occlusion space} of an object $i$ at pose $s_i$ can be defined:
$\tilde{O}_i(s_i)\equiv
O_i(s_i)\cap(\bigcup_{j\neq i}\overline{O_j(s_j)})$, where $\overline{O_j(s_j)}$ denotes the complement of set $O_j(s_j)$.
In other words, it is the region where the camera ray does not pass through other objects after object $i$.
Then the part of object $j$ directly occluded by object $i$ can be defined as $\tilde{O}_{ij}(s_i,s_j)\equiv\tilde{O}_i(s_i)\cap X_j(s_j)$.
Similar as before, object $i$ at pose $s_i$ is defined to \textbf{directly occlude} $j$ at pose $s_j$ if $\tilde{O}_{ij}(s_i,s_j)\neq\emptyset$.
The differences between occlusion space and direct occlusion space of an object are illustrated in Fig. \ref{fig:odg}.

Building on the above-defined notions, a directional graph can be built in order to model object extraction dependencies, which is named the \textbf{Occlusion Dependency Graph (ODG)}.
In this graph, each node $i$ represents the object $i$. Two nodes $i$ and $j$ are connected by edge $(j,i)$, if $\tilde{O}_{ij}(s_i,s_j)\neq\emptyset$,
i.e. if object $i$ directly occludes $j$.
An example of the ODG is shown in Fig. \ref{fig:odg}, where a node $C$ is included to represent the camera, and an edge $(i, C)$ exists if part of object $i$ is visible from the camera.
\begin{figure}[ht]
    \centering
     \begin{subfigure}[t]{0.31\textwidth}
        \centering
        \includegraphics[width=0.98\textwidth]{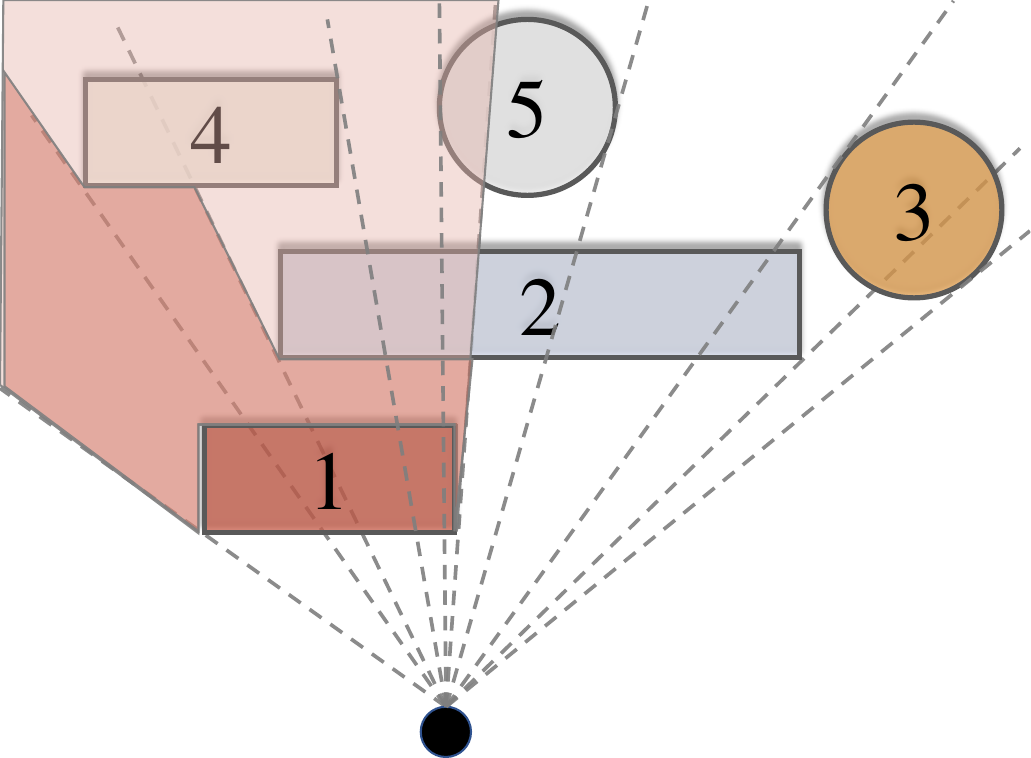}
         \caption{}
         \label{fig:odg-1}
     \end{subfigure}\quad
     \begin{subfigure}[t]{0.31\textwidth}
         \centering
        \includegraphics[width=0.98\textwidth]{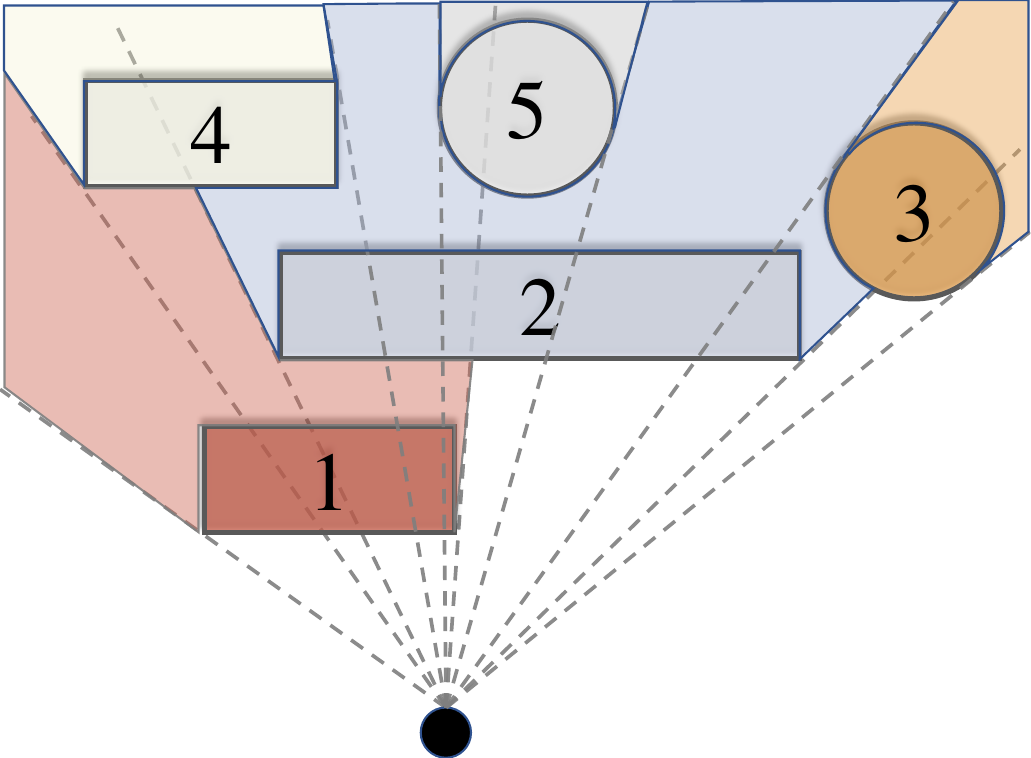}
        \caption{}
         \label{fig:odg-2}
     \end{subfigure}\quad 
     \begin{subfigure}[t]{0.31\textwidth}
         \centering
        \includegraphics[width=0.7\textwidth]{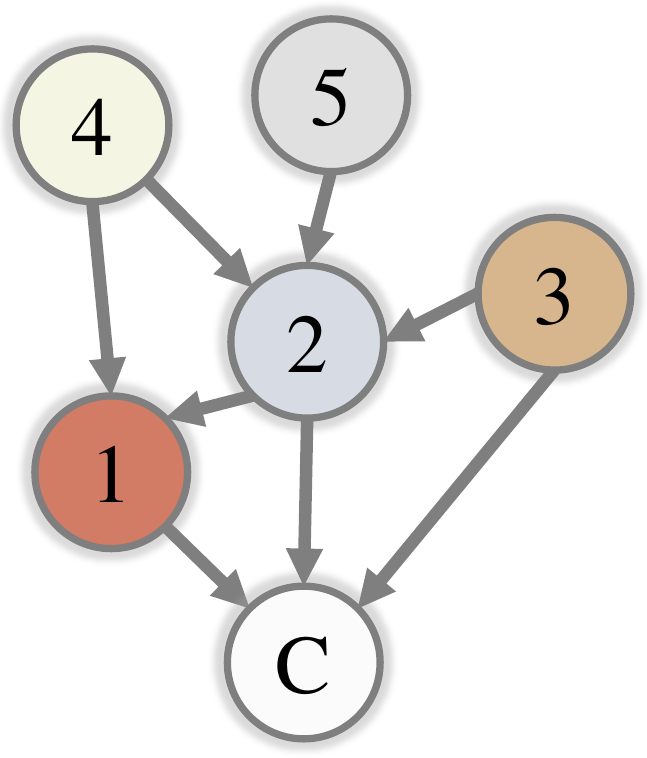}
        \caption{}
        \label{fig:odg-3}
     \end{subfigure}    
    \caption{Illustration of ODG. (a): Occlusion region and direct occlusion region of object 1. The dark red region indicates the direct occlusion region, while the occlusion region includes the light and the dark regions.
    (b): Direct occlusion region of each object. Each occlusion region is shown using a similar color as the corresponding object.
    (c): ODG of the objects. An edge $(i,j)$ exists if object $j$ directly occludes object $i$. $C$ indicates the camera.}
    \label{fig:odg}
    \vspace*{-5mm}
\end{figure}
Safety constraints require the topological order of ODG to be followed so that only fully revealed objects are extracted first.
The existence of such an order is guaranteed when the ODG is a Directed Acyclic Graph (DAG), the condition of which is shown in Lemma \ref{lemma:convex}\footnote{The proof can be found in the Appendix: \url{https://arxiv.org/abs/2205.11719}.}.
\vspace*{-1mm}
\begin{lemma}
\label{lemma:convex}
Define the relation $A\prec B$ if there exists $a$ on the 2D projection of object $A$ and $b$ on the projection of $B$ such that $(c,a,b)$ forms a straight line, where $c$ indicates the camera 2D projection. Then if the top-down projections of objects are convex and do not intersect with each other, objects are acyclic with the defined relation.
\end{lemma}
\vspace*{-1mm}
If object $A$ occludes $B$, it must hold that $A\prec B$. Thus, if the ODG contains a cycle, there exists a cycle for the relation "$\prec$". Hence for objects whose top-down projections are convex and do not intersect, the ODG is a DAG.
In practice, the ODG can only be gradually revealed from partial observations but not fully constructed from the initial observation. Nevertheless, the topological order of ODG can still be followed by image cues even without ODG\footnote{See Appendix for a detailed discussion: \url{https://arxiv.org/abs/2205.11719}}. Having a topological order is not sufficient, as other objects can block extraction, and in-place relocation\footnote{In-place relocation refers to rearrangement of objects within the workspace, i.e., not utilizing buffer space for the rearrangement that is external to the workspace.} can modify the scene.
To jointly tackle these challenges, while following the topological order of the ODG, a theoretical framework is introduced in the following section.

\subsection{A Probabilistically Complete Algorithm}
The above challenges can be decoupled into the reachability and visibility challenges. The former involves finding valid object poses and motion planning of the robot, while the latter involves reconstructing objects and revealing occlusion.
The following primitives can be introduced to tackle them individually:
\begin{enumerate}
\item $reconstruct(i)$: this primitive extracts fully-revealed object $i$, reconstructs it, and places it back to its initial pose.
\item $reach\_rearrange(V,i)$: this primitive rearranges objects in set $V$ so that object $i$ becomes reachable, and can be safely extracted.
\item $vis\_rearrange(V,O)$: this primitive rearranges objects in set $V$ so that part of the occlusion space $O$ becomes revealed, and returns the revealed space.
\end{enumerate}
These primitives are general and can be implemented differently in practice. Building upon them, a framework can be obtained as shown in Algorithm \ref{alg:complete-algo}.
This algorithm only reconstructs fully revealed objects and relocates objects after they are reconstructed. Hence it follows the topological order of ODG and ensures safety.
The algorithm ensures an object $i$ is extractable by $reach\_rearrange(V^*,i)$ where $V^*$ is the set of reconstructed objects. It makes sure the direct occluded space of object $i$ is revealed by $vis\_rearrange(V^*,\tilde{O}_i(s_i^0))$, where $\tilde{O}_i(s_i^0)$ can be determined based on images.
The algorithm postpones failed actions and attempts them again when more space is revealed. It terminates when all actions fail, indicating that the problem is unsolvable, and avoiding infinite loops that make no progress.

\vspace*{-5mm}
\begin{algorithm}
\DontPrintSemicolon
\begin{small}
\SetKwComment{Comment}{\% }{}
\caption{Target Extraction Under Occlusion Framework}
\label{alg:complete-algo}
\SetAlgoLined
    init action queue $Q=[]$,
    revealed objects $V$,
    and reconstructed objects $V^*=\emptyset$\\
    update action queue: $Q=[reconstruct(i),\forall i\in V]$\\
    \While {$|V^*|\neq n$}
    {
    If all actions in $Q$ failed, then \textbf{terminate} as an infinite loop occurs, and the problem is unsolvable\\
    $action\gets Q.pop()$\\
    \If{$action==reconstruct(i)$}
    {
        \If{$reach\_rearrange(V^*,i)==$FALSE}
        {$Q.push(reconstruct(i))$\\}
        \Else
        {
        $reconstruct(i)$\\
        $V^*\gets V^*\cup\{i\}$\\
        $Q.push(vis\_rearrange(\_,\tilde{O}_i(s^0_i)))$\\
        }
    }
    \If{$action==vis\_rearrange(\_,O)$}
    {
    $O'\gets vis\_rearrange(V^*,O)$\\
    \If{$O'==\emptyset$}
    {
        $Q.push(vis\_rearrange(\_,O))$\\
    }
    \Else
    {
        $V'\gets $obtain newly revealed objects with revealed occlusion space\\
        $Q.push(reconstruct(i))\text{ }\forall i\in V'$\\
        $V\gets V\cup V'$\\
        $Q.push(vis\_rearrange(\_,O\setminus O'))$\\
    }
    }
}   
\end{small}
\end{algorithm}
\vspace*{-5mm}

\begin{lemma}
The algorithm defined in Algorithm \ref{alg:complete-algo} is probabilistically complete (and interchangeably resolution-complete).
\end{lemma}
\vspace*{-1mm}
\begin{proof}
If the algorithm finds a solution, the solution is correct, as it follows the topological order of ODG and reveals objects and occlusion safely.
Conversely, if a solution exists, it can be converted to a sequence of the primitives mentioned above.
Denote $(i_j,s_j,\tau_j)$ as the $j$-th action in the solution, and denote the index of the first action applied to object $i$ as $\pi(i)$.
Since the solution follows the topological order of ODG, it can be reformulated as:
\begin{align*}
[&(i_{\pi(1)},s_{\pi(1)},\tau_{\pi(1)}),
(i_{\pi(1)+1},s_{\pi(1)+1},\tau_{\pi(1)+1}),
\dots,\\
&(i_{\pi(2)},s_{\pi(2)},\tau_{\pi(2)}),
(i_{\pi(2)+1},s_{\pi(2)+1},\tau_{\pi(2)+1})
,\dots,(i_{\pi(n)},s_{\pi(n)},\tau_{\pi(n)})
]
\end{align*}
As the solution is safe, $\{i_{\pi(k)+1},i_{\pi(k)+2},\dots,i_{\pi(k+1)-1}\}$ are objects that are reconstructed before object $i_{\pi(k+1)}$, hence only includes objects in $\{i_{\pi(1)},\dots,i_{\pi(k)}\}$.
Notice that $(i_{\pi(k)},s_{\pi(k)},\tau_{\pi(k)})$ can be converted into two actions:
\[(i_{\pi(k)},s^0(i_{\pi(k)}),\tau^1_{\pi(k)}),(i_{\pi(k)},s_{\pi(k)},\tau^2_{\pi(k)}),\]
which puts object $i_{\pi(k)}$ back to its initial pose after reconstruction, and then to pose $s_{\pi(k)}$. The first action is equivalent to primitive $reconstruct(i_{\pi(k)})$. Hence
\[[(i_{\pi(k)},s_{\pi(k)},\tau_{\pi(k)})]\longrightarrow
[reconstruct(i_{\pi(k)}),\tau^1_{\pi(k)}),(i_{\pi(k)},s_{\pi(k)},\tau^2_{\pi(k)})].\]
After the conversion, as an object must be reachable before extracted, a dummy action
$reach\_rearrange(V^*,i_{\pi(k)})$ can be inserted before each $reconstruct(i_{\pi(k)})$ action. The action is dummy as the object is already reachable and nothing needs to be done. The object set $V^*$, as argued before, includes 
all reconstructed objects in 
$\{i_{\pi(1)},i_{\pi(2)},\dots,i_{\pi(k-1)}\}$. Hence
\[[reconstruct(i_{\pi(k)})]\longrightarrow
[reach\_rearrange(V^*,i_{\pi(k)}), reconstruct(i_{\pi(k)})].\]
Then any action that reveals part of $\tilde{O}_i(s_i^0)$ for some object $i$ can be changed to $vis\_rearrange(V^*,\tilde{O}_i(s_i^0)$, where $V^*$ includes fully reconstructed objects at that time.
For simplicity of analysis, it is safe to assume $reconstruct(i)$ only reconstructs object $i$ but does not reveal occlusion, adding an extra action afterward.
In such cases, $vis\_rearrange(V^*,\tilde{O}_i(s_i^0))$ must be after $reconstruct(i)$.

After applying the above conversions, there may still be unconverted actions that do not reveal any part of occlusion.
Since these actions must be applied to a reconstructed object at that time, it is safe to merge them to the next $reach\_rearrange(\cdot,\cdot)$ or $vis\_rearrange(\cdot,\cdot)$ primitive.
After the above steps, the solution is changed to a sequence of primitives, which can be found by Algorithm \ref{alg:complete-algo} probabilistically and asymptotically as the number of iterations increases.
Hence the algorithm is probabilistically complete.
\end{proof}
The framework \ref{alg:complete-algo} is probabilistically complete but depends on two rearrangement primitives, which may run infinitely. The time complexity of the resolution-complete version of the framework is included in the Appendix section.
A monotone algorithm is introduced in the following section to handle practical cases.

\subsection{A Monotone Algorithm}
\begin{figure}[ht]
    \centering    \includegraphics[width=0.99\textwidth]{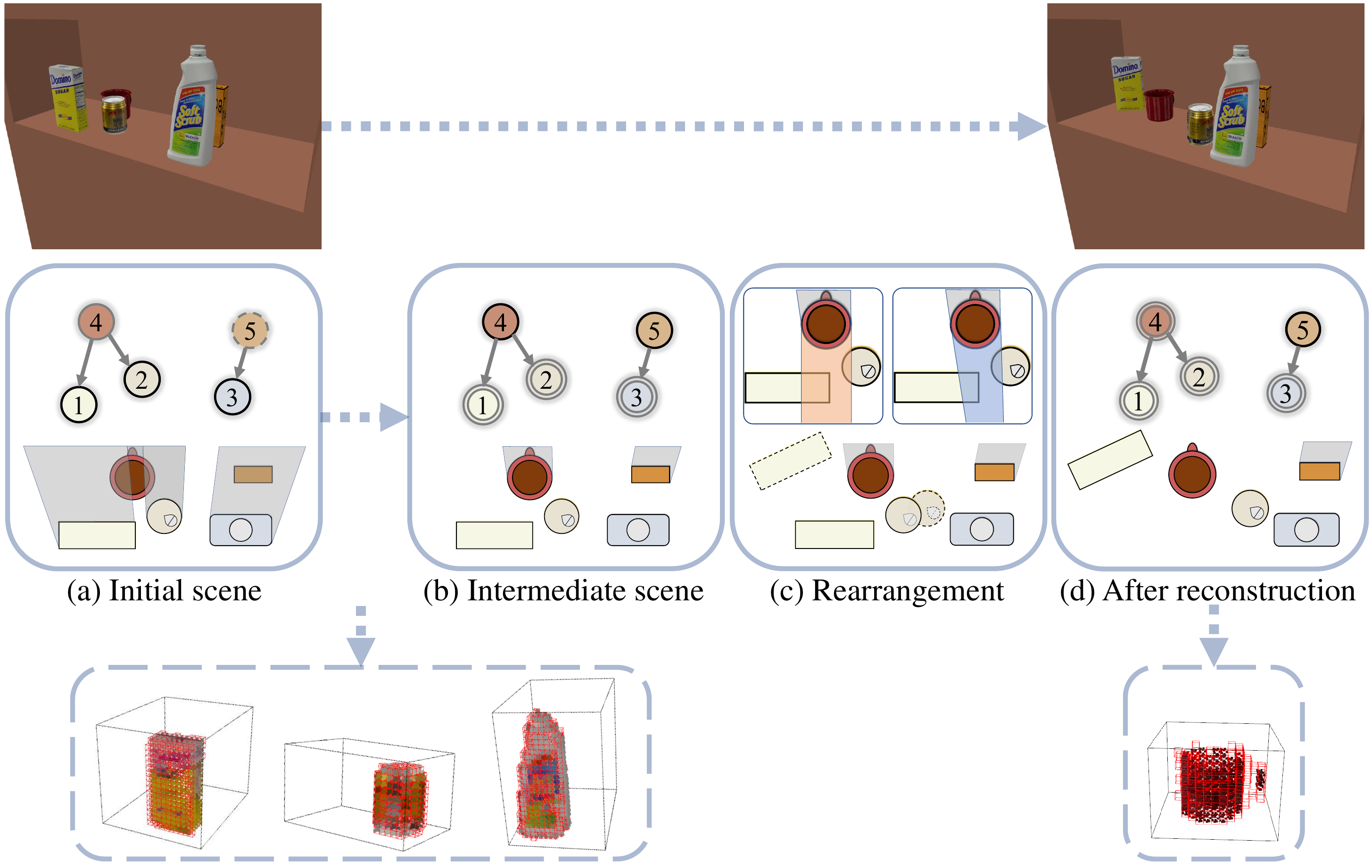}
    \caption{One iteration of the monotone algorithm. The 2D projection of objects and the ODG are illustrated. In (a), objects 1, 2, and 3 are fully revealed and highlighted in ODG. 5 is hidden and hence is dashed. 4 is partially occluded. After reconstruction of 1, 2, and 3, in (b), they are marked by double circles in ODG. 4 and 5 become fully revealed and hence highlighted. Before extracting 4, the sweeping volume shown in red in (c) and the vision cone shown in blue are visualized. The blocking objects 1 and 2 are rearranged by sampling placements. Then 4 is extracted and reconstructed, with the occlusion reduced as in (d).
    }
    \label{fig:method-pipeline}
\vspace*{-5mm}
\end{figure}


There are two primitives which may require repeated moves of some objects: 
$vis\_rearrange(\cdot,\cdot)$ when multiple calls are required to reveal an occluded region; and when object rearrangement requires repeated moves of objects.
The monotonicity can thus be defined where each directly occluded region can be revealed by one $vis\_rearrange(\cdot,\cdot)$ action, and object rearrangement moves each object at most once.
Algorithm \ref{alg:fast-algo} solves the case when both subproblems are monotone. 
The algorithm rearranges reconstructed objects so that the object to be extracted and reconstructed next is not blocked or visually occluded by others.
Thus, the direct occlusion space can be fully revealed after extracting the object.
As both reachability and visual occlusion are tackled in the rearrangement primitive,
the algorithm defines a
primitive $rearrange(i)$ that combines $reach\_rearrange(V,i)$ and $vis\_rearrange(V,O)$.

As illustrated in Figure \ref{fig:method-pipeline} step (c), during rearrangement, the algorithm computes the sweeping volume and the vision cone in front of the object to check blocking objects and visual occlusion. If objects that have not been reconstructed exist in this region, then $rearrange(i)$ fails as they cannot be safely moved. Otherwise, the primitive rearranges all reconstructed objects in this region. It may also move other reconstructed objects outside this region to open space for placements.
After that, It samples and verifies placements to avoid the computed region.
Finally, a trajectory is found by motion planning and executed.

The primitive $reconstruct(i)$ decomposes the trajectory into a suction phase, an extraction phase, a sensing phase, and a placement phase. The suction phase computes a trajectory to move the manipulator to a valid suction pose. The extraction phase then extracts the object out of the shelf. The sensing phase completely reconstructs the object by a sequence of in-hand moves. Finally, the placement phase places the object back to its initial pose and resets the robot.
The algorithm guarantees that $\tilde{O}_i(s^0_i)$ is revealed when object $i$ is reconstructed. Then following the topological order of ODG, it reconstructs all objects safely and reveals the entire scene.
The algorithm is probabilistically complete for monotone cases following a similar proof. However, experiments show it can solve many practical tasks and perform better than the constructed baselines.
\vspace*{-5mm}
\begin{algorithm}
\DontPrintSemicolon
\begin{small}
\SetKwComment{Comment}{\% }{}
\caption{Target Extraction Under Occlusion Monotone Algorithm}
\label{alg:fast-algo}
\SetAlgoLined
    initialize reconstructed objects $\bar{V}=\emptyset$, failure count $N_F=0$\\
    \While {$|\bar{V}|\neq n$}
    {
    visible objects $S$, revealed objects $V$, occlusion region $\hat{O}\gets$ perception()\\
    $i\gets (V\setminus\bar{V})[0]$\\
    rearrange($i$)\\
    status, newly revealed object set $\tilde{V}$, updated occlusion region $\hat{O}$, updated object model $\hat{X}_i$ $\gets$ reconstruct($i$)\\
    \If {status $==$ FAILURE}
    {
        $V \gets V[1:]\cup \{i\}$, 
        $N_F\gets N_F+1$\\
    }
    \Else
    {
        reset failure count $N_F\gets0$\\
        $V\gets V\cup\tilde{V}$, 
        $\bar{V}\gets\bar{V}\cup\{i\}$\\
    }
    
    \If {$N_F==|V\setminus\bar{V}|$}
    {
        \Return FAILURE\\
    }
    }
\end{small}
\end{algorithm}
\vspace*{-7mm}

\section{System Design}
The system setup uses the left arm of a Yaskawa Motoman robot with a suction gripper. It is an 8 DOF kinematic chain including the torsional joint, with a static RGBD camera. Voxels are used to model the workspace. Each voxel can be free, occupied, or occluded. The occlusion space of each object is computed given the depth and segmentation images by:
$\hat{O}_i(s_i,t)=
\{x:\phi(x)\in\phi(X_i(s_i)),d(x)\geq depth(\phi(x))\}$.
The net occlusion up to time $t$ is obtained through the intersection of occlusion at each time. Each object is modeled via voxels and TSDF for reconstruction \cite{newcombe2011kinectfusion}.
Lastly, objects' pixel positions and depth values are used to determine their occlusion relationship, which is sufficient for convex objects
\footnote{
The details are included in the Appendix: \url{https://arxiv.org/abs/2205.11719}.}.
The occlusion caused by the robot is filtered out using the ground-truth segmentation image obtained from the simulator and is not used for occlusion update.

To sample viable placements in $rearrange(i)$, the system uses a novel method by convolution. It first projects the collision space and object geometry to 2D maps, then computes a convolution of them to obtain collision-free 2D placement poses. Samples are then drawn from this set and validated by IK and collision check.
The system reconstructs an object by moving it to several waypoints outside the shelf. Waypoints are sampled such that the information gain (info-gain), computed as the volume of unseen space of the object, is maximized, similar to \cite{huang2019building}. Fig. \ref{fig:method-reconstruct} illustrates the process, where the pose with the max info-gain is selected as the next waypoint.

\begin{figure}[ht]
    \centering    \includegraphics[width=0.99\textwidth]{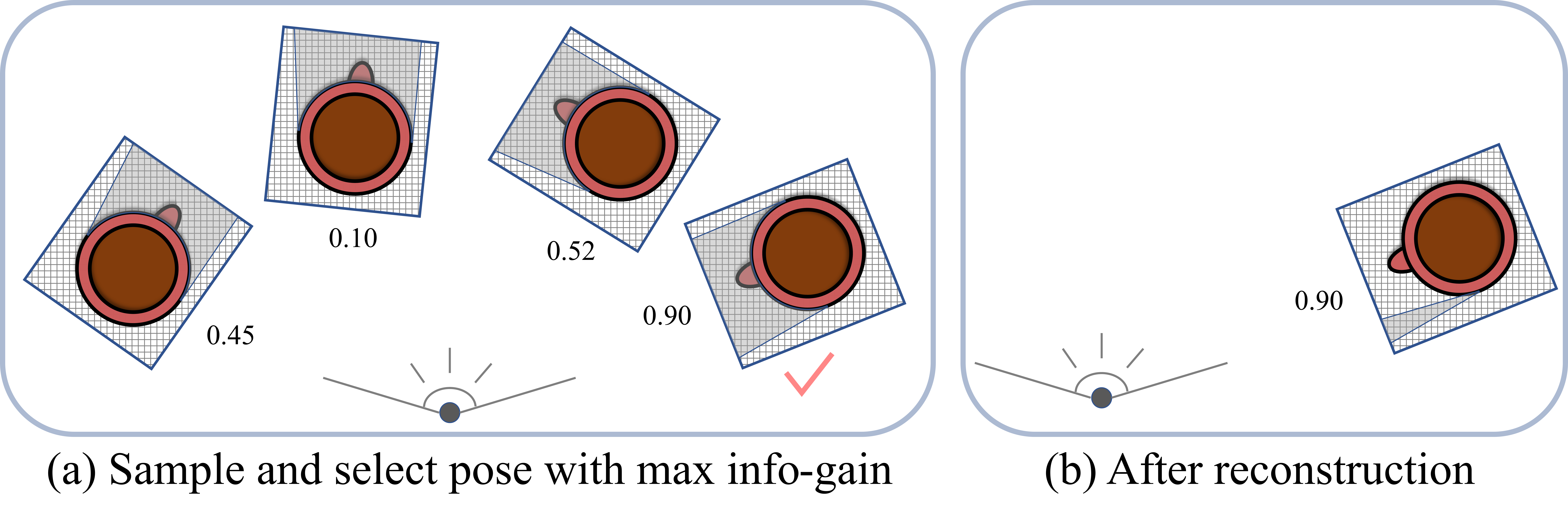}
    \vspace{-.15in}
    \caption{Object reconstruction by sampling poses and computing "info-gain". Sensing is done at several waypoints, each obtained through the process shown in the figure. Potential poses are first sampled with info-gain computed. Then the pose with the max info-gain is used as the next waypoint for sensing.}
    \label{fig:method-reconstruct}
    \vspace*{-5mm}
\end{figure}

The system uses the PyBullet physics engine \cite{coumans2021} to model the scene. The RGBD and segmentation images captured by the simulated camera are used for perception. Motion planning is achieved through MoveIt \cite{coleman2014reducing}.

\section{Experiments}
\subsection{Experimental Setup}
A random and a greedy baseline are constructed for comparison. The random baseline uniformly samples objects to pick and locations to place them. The greedy one follows prior work \cite{gupta2013interactive} by sampling choices of objects and placements and picking the action with the max info-gain, which is the volume of revealed occlusion. Both methods ensure safety by moving only fully revealed or reconstructed objects and reconstructing objects before other actions. The greedy baseline prioritizes reconstructing unreconstructed objects if it is safe to do so. As the methods can be running for an infinite amount of time due to randomness, they are terminated once all objects are reconstructed or a timeout is reached. Hence the number of objects is given only in both baselines but not available in the proposed algorithm. The three algorithms are compared using  a PyBullet simulated environment with random rectangular prisms and cylinders on the shelf, similar to the setting in prior work \cite{huang2022mechanical}.
Experiments include three scenarios with 5, 8, or 11 objects. Each scenario has three difficulty levels defined according to the size of objects. Objects are considered small if the width and length scale is 4cm or 5cm, medium if 7cm or 8cm, and large if 10cm or 11cm. Heights are randomly picked for each object from 9cm, 12cm, and 15cm. The easy level includes all small objects; the medium level includes 60\% small, 30\% medium, and 10\% large objects; the hard level includes 25\% small, 50\% medium, and 25\% large objects. Each scenario contains 10 randomly generated problem instances with at least one fully-hidden object. The instances are fixed across algorithms to allow fair comparison. A total of 5 trials of each algorithm is executed for each instance to account for the randomness in the algorithm. Hence in total, there are 90 instances, with 450 total trials for each algorithm.

Easy, medium, and hard levels are denoted by indices 1, 2, and 3. Experiments are denoted as $n$-$k$ for $n$ objects in the scene with difficulty level $k$. Due to a large number of trials, the timeout is chosen to be 250s for 5 objects; 350s, 400s, and 450s for 8-1, 8-2, and 8-3; and 500s, 600s, and 700s for 11-1, 11-2 and 11-3. In particular, it increases linearly as the difficulty level increases. Each trial is considered a success if all objects are reconstructed before timeout and failure otherwise. The random and greedy algorithms fail to reveal the occlusion volume even when all objects are reconstructed. In contrast, the proposed algorithm fully explores the environment at termination in most trails, which is illustrated in the remaining occlusion volume shown in Fig. \ref{fig:exp-performance}.
Experiments also include YCB objects \cite{calli2015ycb} obtained from prior work \cite{liu2021ocrtoc} with 5 or 8 random objects. As the objects are larger than objects in simple geometry experiments, the scene with 11 objects is more cluttered and harder to solve, taking significantly more time.  Hence YCB experiments with 11 objects are omitted and left for future work after the system is optimized. For each number of objects, there are 10 instances, and each instance is executed for 5 trials\footnote{Code and videos: \url{sites.google.com/scarletmail.rutgers.edu/occlusion-manipulation}.}.
All experiments were executed on an AMD Ryzen 7 5800X 8-Core Processor with 31.2GB memory.

\subsection{Results and Analysis}
As shown in table \ref{exp:table-success}, the proposed algorithm achieves a near 100\% success rate in all scenarios, consistently better than random and greedy algorithms given a large number of trials. The running time and number of actions of the three algorithms are provided in Fig. \ref{fig:exp-performance}. The number of actions is at least the number of objects in each scene  but may be more if objects need multiple moves for rearrangement. The proposed method needs less time to solve problems and fewer actions. It fully reveals the scene, consistently outperforming the two baselines.

\begin{table}[t]
\centering
\caption{Success rate across scenarios.
For tasks with rectangular prisms and cylinders,
a total of 450 trials are run for each algorithm.
For tasks with YCB objects, a total of 100 trials are run for each algorithm.}
\begin{tabular}{|c|c|c|c|c|c|c|c|c|c|c|c|}
\hline
Algo & 5-1 & 5-2 & 5-3 & 8-1 & 8-2 & 8-3 & 11-1 & 11-2 & 11-3 & ycb-5 & ycb-8  \\ \hline
random & 68\% & 68\% & 54\% & 26\% & 16\% & 16 \% & 0 \% & 2\% & 8\% & 46\% & 10\%
\\ \hline
greedy & 86\% & 68\% & 70\% & 38\% & 40\% & 36\% & 22\% & 14\% & 24\% & 84\% & 46\%
\\ \hline
ours & \textbf{100\%} & \textbf{100\%} & \textbf{100\%} & \textbf{100\%} & \textbf{100\%} & \textbf{94\%} & \textbf{100\%} & \textbf{100\%} & \textbf{98\%} &
\textbf{96\%} &
\textbf{90\%}
\\ \hline
\end{tabular}
\vspace*{-7mm}
\label{exp:table-success}
\end{table}

The proposed algorithm's implementation is not optimized and is left as future work. The latency of each component,
including perception, motion planning, pose sampling, and rearrangement has been measured. Rearrangement takes more time as difficulty increases, with its time ratio ranging from 9\% on scenario 5-1 to 30\% on scenario 11-3 . The remaining time is spent mainly on motion planning and ROS communication, both taking roughly 30\%. Pose generation and perception take 23\% and 17\% of the time without rearrangement. The number of calls to each component increases as the number of objects increases. Given this, the average time of each component can be derived, with 0.18s on average for perception, 0.86s for motion planning, 0.78s for pose generation, and 0.13s for ROS communication. Depending on the number of objects, each rearrangement call takes 10s to 37s across the experiments. Failure cases are often caused by timeout due to unsuccessful motion planning queries and rearrangement routines. Therefore, more efficient implementations of the motion planning and rearrangement components will improve performance.
\vspace*{-5mm}
\begin{figure}[H]
     \centering
    \includegraphics[width=0.7
    \textwidth]{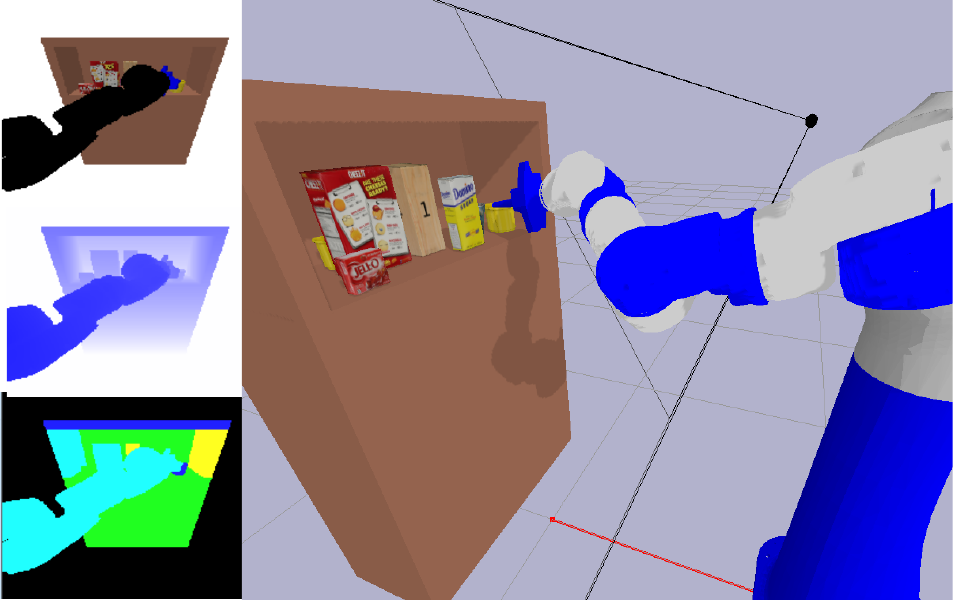}
    \vspace{-.1in}
     \caption{Example setup for 8 YCB objects. }
     \label{fig:scene-setup}
\vspace*{-10mm}
\end{figure}

\begin{figure}[H]
        \centering
        \includegraphics[width=0.97\textwidth]{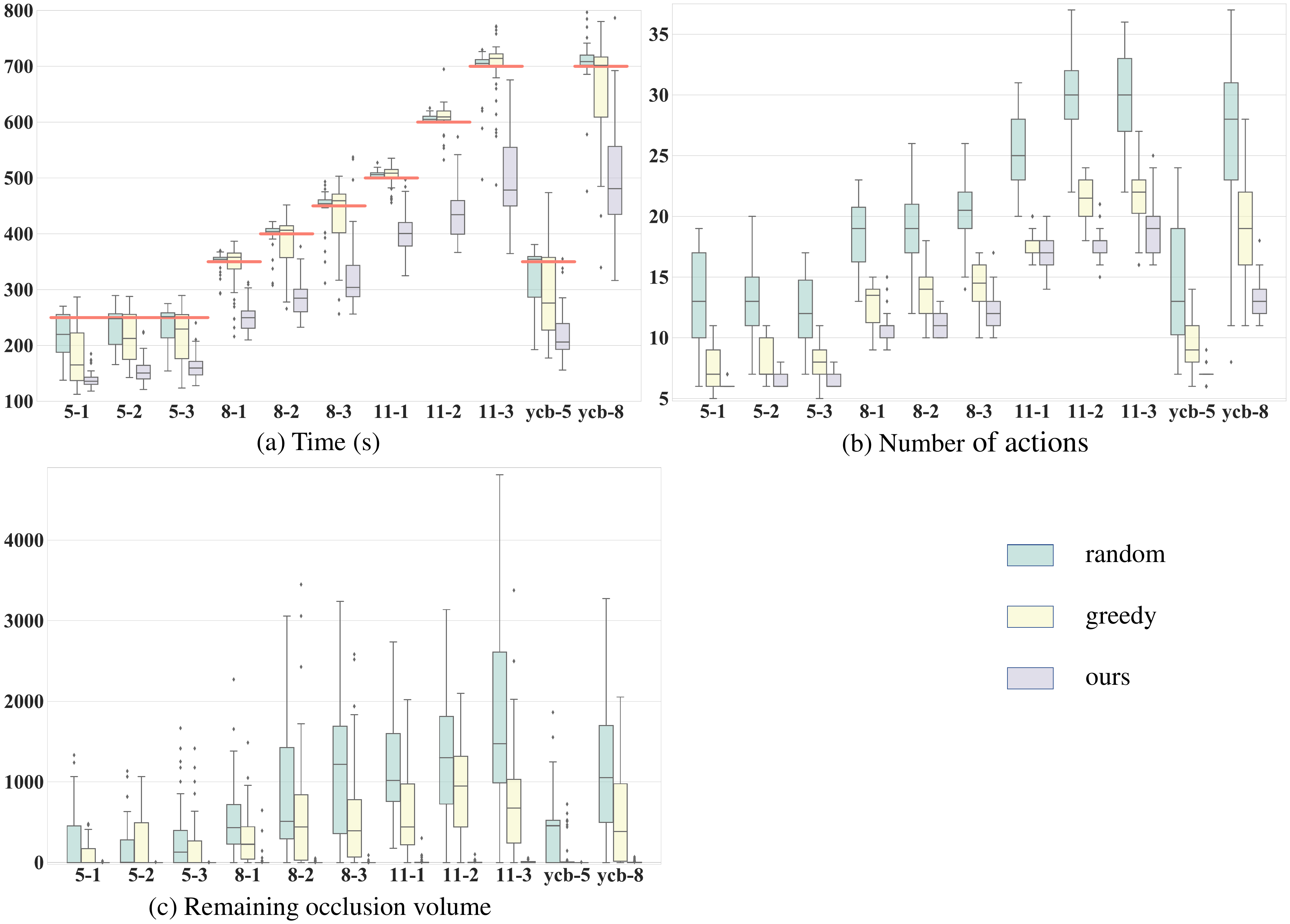}
            \vspace{-.1in}
        \caption{Box plots showing running time (a), number of actions (b), and the remaining occlusion volume (c) of the algorithms. The red line in (a) indicates the predetermined timeout for each scenario to compute success ratio in Table 1.}
    \label{fig:exp-performance}
\vspace*{-5mm}
\end{figure}
\vspace*{-5mm}

\section{Conclusion and Future Work}
This work first provides a classification of occlusion-aware manipulation tasks and then focuses on the safe object reconstruction in confined spaces with in-place rearrangement. This work defines the notion of monotonicity for this challenge. It proposes a probabilistically complete framework with an efficient algorithm instantiation for monotone instances.  Experiments in simulation show that the proposed method attains a high success rate with less computational time and fewer actions than a random and a greedy baseline. As the current work focuses on a foundational understanding of the domain, the optimization of each primitive component is left for future work. A more informed search can also optimize the algorithm to reduce the number of actions needed. Furthermore, this work will be integrated with perception primitives for object pose tracking and reconstruction to address real-world platforms.

%
%

\bibliographystyle{abbrv} 
\bibliography{main} 

\begin{thebibliography}{10}

\bibitem{ahn2021integrated}
J.~Ahn, J.~Lee, S.~H. Cheong, C.~Kim, and C.~Nam.
\newblock An integrated approach for determining objects to be relocated and
  their goal positions inside clutter for object retrieval.
\newblock In {\em ICRA}, pages 6408--6414, 2021.

\bibitem{bejjani2021occlusion}
W.~Bejjani, W.~C. Agboh, M.~R. Dogar, and M.~Leonetti.
\newblock Occlusion-aware search for object retrieval in clutter.
\newblock In {\em IROS}, pages 4678--4685, 2021.

\bibitem{bohg2017interactive}
J.~Bohg, K.~Hausman, B.~Sankaran, O.~Brock, D.~Kragic, S.~Schaal, and G.~S.
  Sukhatme.
\newblock Interactive perception: Leveraging action in perception and
  perception in action.
\newblock {\em IEEE TRO}, 33(6):1273--1291, 2017.

\bibitem{calli2015ycb}
B.~Calli, A.~Singh, A.~Walsman, S.~Srinivasa, P.~Abbeel, and A.~M. Dollar.
\newblock The ycb object and model set: Towards common benchmarks for
  manipulation research.
\newblock In {\em ICAR}, pages 510--517, 2015.

\bibitem{cheong2020relocate}
S.~H. Cheong, B.~Y. Cho, J.~Lee, C.~Kim, and C.~Nam.
\newblock Where to relocate?: Object rearrangement inside cluttered and
  confined environments for robotic manipulation.
\newblock In {\em ICRA}, pages 7791--7797, 2020.

\bibitem{coleman2014reducing}
D.~Coleman, I.~Sucan, S.~Chitta, and N.~Correll.
\newblock Reducing the barrier to entry of complex robotic software: a moveit!
  case study.
\newblock {\em arXiv:1404.3785}, 2014.

\bibitem{coumans2021}
E.~Coumans and Y.~Bai.
\newblock Pybullet, a python module for physics simulation for games, robotics
  and machine learning.
\newblock \url{http://pybullet.org}, 2016--2021.

\bibitem{danielczuk2019mechanical}
M.~Danielczuk, A.~Kurenkov, A.~Balakrishna, M.~Matl, D.~Wang,
  R.~Mart{\'\i}n-Mart{\'\i}n, A.~Garg, S.~Savarese, and K.~Goldberg.
\newblock Mechanical search: Multi-step retrieval of a target object occluded
  by clutter.
\newblock In {\em ICRA}, pages 1614--1621, 2019.

\bibitem{demaine2010circle}
E.~D. Demaine, S.~P. Fekete, and R.~J. Lang.
\newblock Circle packing for origami design is hard.
\newblock {\em arXiv:1008.1224}, 2010.

\bibitem{dogar2014object}
M.~R. Dogar, M.~C. Koval, A.~Tallavajhula, and S.~S. Srinivasa.
\newblock Object search by manipulation.
\newblock {\em Autonomous Robots}, 36(1):153--167, 2014.

\bibitem{gao2021fast}
K.~Gao, D.~Lau, B.~Huang, K.~E. Bekris, and J.~Yu.
\newblock Fast high-quality tabletop rearrangement in bounded workspace.
\newblock In {\em ICRA}, 2022.

\bibitem{garrett2021integrated}
C.~R. Garrett, R.~Chitnis, R.~Holladay, B.~Kim, T.~Silver, L.~P. Kaelbling, and
  T.~Lozano-P{\'e}rez.
\newblock Integrated task and motion planning.
\newblock {\em Annual Review of Control, Robotics, and Autonomous Systems},
  4:265--293, 2021.

\bibitem{garrett2020online}
C.~R. Garrett, C.~Paxton, T.~Lozano-P{\'e}rez, L.~P. Kaelbling, and D.~Fox.
\newblock Online replanning in belief space for partially observable task and
  motion problems.
\newblock In {\em ICRA}, pages 5678--5684, 2020.

\bibitem{gupta2013interactive}
M.~Gupta, T.~R{\"u}hr, M.~Beetz, and G.~S. Sukhatme.
\newblock Interactive environment exploration in clutter.
\newblock In {\em IROS}, pages 5265--5272, 2013.

\bibitem{huang2022mechanical}
H.~Huang, M.~Danielczuk, C.~M. Kim, L.~Fu, Z.~Tam, J.~Ichnowski, A.~Angelova,
  B.~Ichter, and K.~Goldberg.
\newblock Mechanical search on shelves using a novel “bluction” tool.
\newblock In {\em 2022 International Conference on Robotics and Automation
  (ICRA)}, pages 6158--6164. IEEE, 2022.

\bibitem{huang2020mechanical}
H.~Huang, M.~Dominguez-Kuhne, V.~Satish, M.~Danielczuk, K.~Sanders,
  J.~Ichnowski, A.~Lee, A.~Angelova, V.~Vanhoucke, and K.~Goldberg.
\newblock Mechanical search on shelves using lateral access x-ray.
\newblock In {\em IROS}, pages 2045--2052, 2020.

\bibitem{huang2019building}
K.~Huang and T.~Hermans.
\newblock Building 3d object models during manipulation by reconstruction-aware
  trajectory optimization.
\newblock {\em arXiv:1905.03907}, 2019.

\bibitem{kumar2021graph}
K.~N. Kumar, I.~Essa, and S.~Ha.
\newblock Graph-based cluttered scene generation and interactive exploration
  using deep reinforcement learning.
\newblock In {\em ICRA}, 2022.

\bibitem{kurenkov2020visuomotor}
A.~Kurenkov, J.~Taglic, R.~Kulkarni, M.~Dominguez-Kuhne, A.~Garg,
  R.~Mart{\'\i}n-Mart{\'\i}n, and S.~Savarese.
\newblock Visuomotor mechanical search: Learning to retrieve target objects in
  clutter.
\newblock In {\em IROS}, pages 8408--8414, 2020.

\bibitem{li2016act}
J.~K. Li, D.~Hsu, and W.~S. Lee.
\newblock Act to see and see to act: Pomdp planning for objects search in
  clutter.
\newblock In {\em IROS}, pages 5701--5707, 2016.

\bibitem{liu2021ocrtoc}
Z.~Liu, W.~Liu, Y.~Qin, F.~Xiang, M.~Gou, S.~Xin, M.~A. Roa, B.~Calli, H.~Su,
  Y.~Sun, et~al.
\newblock Ocrtoc: A cloud-based competition and benchmark for robotic grasping
  and manipulation.
\newblock {\em IEEE RA-L}, 7(1):486--493, 2021.

\bibitem{lu2021online}
S.~Lu, R.~Wang, Y.~Miao, C.~Mitash, and K.~Bekris.
\newblock Online model reconstruction and reuse for lifelong improvement of
  robot manipulation.
\newblock In {\em ICRA}, 2022.

\bibitem{mitash2020task}
C.~Mitash, R.~Shome, B.~Wen, A.~Boularias, and K.~Bekris.
\newblock Task-driven perception and manipulation for constrained placement of
  unknown objects.
\newblock {\em IEEE RA-L}, 5(4):5605--5612, 2020.

\bibitem{nam2021fast}
C.~Nam, S.~H. Cheong, J.~Lee, D.~H. Kim, and C.~Kim.
\newblock Fast and resilient manipulation planning for object retrieval in
  cluttered and confined environments.
\newblock {\em IEEE TRO}, 37(5):1539--1552, 2021.

\bibitem{newcombe2011kinectfusion}
R.~A. Newcombe, S.~Izadi, O.~Hilliges, D.~Molyneaux, D.~Kim, A.~J. Davison,
  P.~Kohi, J.~Shotton, S.~Hodges, and A.~Fitzgibbon.
\newblock Kinectfusion: Real-time dense surface mapping and tracking.
\newblock In {\em ISMAR}, pages 127--136, 2011.

\bibitem{novkovic2020object}
T.~Novkovic, R.~Pautrat, F.~Furrer, M.~Breyer, R.~Siegwart, and J.~Nieto.
\newblock Object finding in cluttered scenes using interactive perception.
\newblock In {\em ICRA}, 2020.

\bibitem{price2019inferring}
A.~Price, L.~Jin, and D.~Berenson.
\newblock Inferring occluded geometry improves performance when retrieving an
  object from dense clutter.
\newblock In {\em ISRR}, 2019.

\bibitem{qureshi2021nerp}
A.~H. Qureshi, A.~Mousavian, C.~Paxton, M.~C. Yip, and D.~Fox.
\newblock Nerp: Neural rearrangement planning for unknown objects.
\newblock In {\em RSS}, 2021.

\bibitem{stilman2007manipulation}
M.~Stilman, J.-U. Schamburek, J.~Kuffner, and T.~Asfour.
\newblock Manipulation planning among movable obstacles.
\newblock In {\em ICRA}, pages 3327--3332, 2007.

\bibitem{wang2022lazy}
R.~Wang, K.~Gao, J.~Yu, and K.~Bekris.
\newblock Lazy rearrangement planning in confined spaces.
\newblock In {\em ICAPS}, 2022.

\bibitem{wang2021efficient}
R.~Wang, Y.~Miao, and K.~E. Bekris.
\newblock Efficient and high-quality prehensile rearrangement in cluttered and
  confined spaces.
\newblock In {\em ICRA}, 2022.

\bibitem{wang2020safe}
R.~Wang, C.~Mitash, S.~Lu, D.~Boehm, and K.~E. Bekris.
\newblock Safe and effective picking paths in clutter given discrete
  distributions of object poses.
\newblock In {\em IROS}, 2020.

\bibitem{xiao2019online}
Y.~Xiao, S.~Katt, A.~ten Pas, S.~Chen, and C.~Amato.
\newblock Online planning for target object search in clutter under partial
  observability.
\newblock In {\em ICRA}, 2019.

\end{thebibliography}
\section*{Appendix}
\subsection*{Proof of DAG Property of ODG}

Safety constraints require the topological order of ODG (Occlusion Dependency Graph) to be followed so that only fully revealed objects are extracted first.
The existence of such an order is guaranteed when the ODG is a Directed Acyclic Graph (DAG). To prove that ODG is a DAG under certain regulations of objects, we retrieve to prove a stronger property in Lemma \ref{lemma:convex-appendix}.

\begin{lemma}
\label{lemma:convex-appendix}
Define the relationship $A\prec B$ if there exists a point $a$ on the 2D projection of object $A$ and a point $b$ on the projection of object $B$ such that $(c,a,b)$ forms a straight line in the order of $c$, $a$ and $b$, where $c$ indicates the camera 2D projection. Then 
\textbf{if the top-down projections of objects are convex shapes and do not intersect with each other}, objects do not form a cycle with the defined relationship.
\end{lemma}

\begin{proof}
It is safe to assume that objects can be separated from the camera by a plane.
The proof uses induction on the number of objects that are acyclic. For base case, if $A^1\prec A^2$ and $A^2\prec A^1$, then there exist lines $(c,a^1_1,a^2_1)$ and $(c,a^2_2,a^1_2)$ where $a^1_1$ and $a^1_2$ ($a^2_1$ and $a^2_2$) are inside the top-down projection of $A^1$ ($A^2$). Due to the convexity of the top-down projection, the projections of $A^1$ and $A^2$ must intersect, which leads to a contradiction.

The induction hypothesis claims that $k$ objects do not form a cycle. Then assume there exists a cycle of length $k+1$ corresponding to order $A^1\prec A^2\prec\dots\prec A^{k+1}\prec A^1$.
Assume the induced points of $A^i\prec A^{i+1}$ for $i<k+1$ are $a^i_2$ and $a^{i+1}_1$, and the induced points for $A^{k+1}\prec A^1$ are $a^{k+1}_2$ and $a^1_1$. Denote line segments $L=(a^2_1,a^2_2,a^3_1,a^3_2,\dots,a^{k+1}_1,a^{k+1}_2)$. Define the projections of a point $a$ on $l$ w.r.t. $c$ as the intersection of line $(c,a)$ with $l$.
Denote the angle of line $(c,a^i_j)$ centered at $c$ by $\theta(a^i_j)$, and w.l.o.g., assume $\theta(a^1_2)<\theta(a^1_1)$. Then the angles must be non-decreasing along $L$; otherwise, there exist line segments that share the same angle range, which constitutes a shorter cycle, thus a contradiction to the induction hypothesis.
Hence the projection of $L$ on $l$ w.r.t. $c$ is $l$.
If $L$ intersects with $l$, the intersection could be on $(a^i_1,a^i_2)$ for some $i$. Then the top-down projections of $A^1$ and $A^i$ intersect.
Alternatively, the intersection may lie on $(a^i_2,a^{i+1}_1)$ for some $i$. Then $A^i\prec A^1\prec A^{i+1}$, hence inducing a shorter cycle. Hence, $L$ and $l$ do not intersect.
Denote $d(a)$ as the distance of $a$ to its projection on $l$ w.r.t. $c$. Since $L$ does not intersect with $l$, $d(a)>0$ for any point $a\in L$. Notice that $d(a^1_1)>d(a^{k+1}_2)>0$, but $a^1_1$ lies on line $l$ hence $d(a^1_1)=0$, leading to a contradiction.
Thus the cycle of length $k+1$ does not exist.
Hence by induction, there does not exist a cycle.
\end{proof}
Notice that if object $A$ occludes object $B$, it must hold that $A\prec B$. Thus if the ODG contains a cycle, there exists a cycle for the relation "$\prec$". Hence for objects whose top-down projections are convex and do not intersect, the ODG is a DAG.

\subsection*{Image Cues for Topological Ordering of ODG}
Lemma \ref{lemma:convex-appendix} can be used to guide the search during online execution of the pipeline, where ground-truth ODG is not available. Here we assume that objects have convex top-down projections and collide iff their top-down projections intersect. One such example is when objects are either cylinders or rectangular prims.

Under this setting, image cues can be used to determine the topological ordering of objects to extract. The process is to find adjacent pixels that belong to two different objects, and one pixel has a smaller depth value (suppose of object $A$) than the other pixel (suppose of object $B$). Then $B$ is assumed to be an ancestor of $A$ in the ODG, and thus $A$ must be extracted before $B$. 
For all objects that are not occluded by others, and thus are the sinks of the ODG, they must have a smaller depth value than their adjacent objects according to Lemma \ref{lemma:convex-appendix}. Hence this process can identify all objects that are sinks in the ODG at each time.
The limitation of this approach is when two objects are adjacent to each other in the image, but do not occlude each other. In such cases, there is no certainty that they do not occlude each other based on partial observation. Hence to ensure safety a conservative estimation is that the object in the front occludes the object in the back at the adjacent pixels.

\subsection*{Time Complexity of Algorithm 1}

The main paper introduces a probablistically complete framework, which can be easily adapted to a resolution complete framework, to ease the analysis of time complexity.
We consider the case when all actions are exhausted to analyze the worst-case time complexity. 
Notice that in the above framework, $vis\_rearrange(\cdot,\cdot)$ may require an infinite number of actions to reveal an occluded space. 
By modeling occlusions by voxels, the number of actions to reveal an occluded space is finite, upper bounded by the volume of the occluded region.
The worst-case time for placement sampling can also be quantified as the number of discrete placements in the workspace.
\textbf{For simplicity, we assume that a monotone object rearrangement solver can solve each rearrangement action. The discussion of non-monotone cases is deemed at future work.}
Define the max number of placements to sample as $M$, which depends on the resolution of the space. Define the max number of actions to reveal an occluded region as $n_v$. Assume one pick-and-place action given the placement pose takes $O(k)$ time.

In the worst case, the solution requires successful primitives of number $n(2+n_v)$ in total. Define a "pass" as one where each action in the action queue is tested. Hence the number of passes is upper bounded by $n(2+n_v)$ when each pass only has one successful primitive.
The bound is not tight, as successful actions do not need to be called again.
Next, we analyze the time in one pass.
For $reconstruct(\cdot)$, $n$ calls are required, where each takes $O(k)$ time.
For $reach\_rearrange(\cdot,\cdot)$, $n$ calls are required, where each takes $O(M^nnk)$ time in the worst case to rearrange $n$ objects and to search for all possible placements.
$vis\_rearrange(\cdot,\cdot)$ also requires $n$ calls where each takes $O(M^nnk)$ time to rearrange objects.
Hence the total time is upper bounded by:
\[O(n(2+n_v)n(k+2M^nnk)=O(n^3n_vkM^n).\]

\subsection*{Extraction Mode of Objects}
Firstly, only fully revealed object is allowed to be extracted or moved, as the unrevealed parts do not have collision-free guarantee. Also,
as the finding of a safe extraction for a general-shaped object is not the target in this work, we simply assume that the initial object shape and pose satisfies the following conditions: the object shape is contained in the sweeping volume of the visible surface along the negative direction of the extraction motion, and the object is convex (otherwise there might be "holes" which can fit hidden objects during moving).
To formalize the condition, define the clearance of a surface $F$ along direction $d$ to be
\[\textrm{Clearance}(F,d)\equiv
\{x+\lambda d:\lambda \geq0, x\in F\}.\]
Similarly, for a volume $X$ we can also define clearance as:
\[\textrm{Clearance}(X,d)\equiv
\{x+\lambda d:\lambda \geq0, x\in X\}.\]
Denote the object coverage of an object $i$ at pose $s_i$ as $X_i(s_i)$, the visible surface as $F_i(s_i)$ and the extraction direction as $d$, then it is easy to see that Clearance$(F_i(s_i),d)\cap X_i(s_i)=\emptyset$.
Then we have
\begin{lemma}
If $X_i(s_i)\subseteq\textrm{Clearance}(F_i(s_i),-d)$, and object $i$ is of convex shape, then 
$\textrm{Clearance}(X_i(s_i),d)=X_i(s_i)\cup
\textrm{Clearance}(F_i(s_i),d)$.
\end{lemma}

\begin{proof}
For any point $x\in X_i(s_i)$, since $X_i(s_i)\subseteq\textrm{Clearance}(F_i(s_i),-d)$, there exists $\lambda_0\geq0$ such that
$x+\lambda_0 d\in F_i(s_i)$.
Hence for any $\lambda\geq0$, if $\lambda<\lambda_0$,
\[x+\lambda d=
(1-\frac{\lambda}{\lambda_0})x+ \frac{\lambda}{\lambda_0}(x+\lambda_0d)
\in X_i(s_i).\]
If $\lambda\geq\lambda_0$,
\[x+\lambda d
=x+\lambda_0 d + (\lambda-\lambda_0)d\in\textrm{Clearance}(F_i(s_i),d).\]
Thus $\textrm{Clearance}(X_i(s_i),d)=X_i(s_i)\cup\textrm{Clearance}(F_i(s_i),d)$.
\end{proof}
This entails that
\begin{enumerate}
    \item the sweeping volume along the negative direction of the extraction motion of visible surface should contain the object geometry (a condition on the object geometry and its pose)
    \item the object geometry should be convex.
\end{enumerate}
The counter examples are included in Figures \ref{fig:safety-condition}, and the discussion about object convexity is included in Figures \ref{fig:safety-condition-convex}.

\begin{figure}[ht]
    \centering
     \begin{subfigure}[t]{0.48\textwidth}
    \centering
    \includegraphics[width=0.95\textwidth]{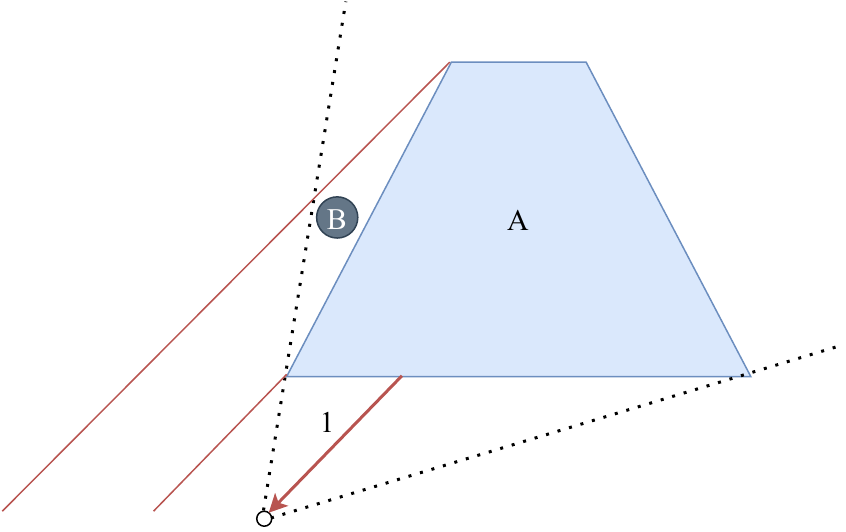}
    \caption{Direction 1 does not satisfy the first condition. $A$ might collide with a hidden object $B$ which cannot be detected given current observation.}
    \label{fig:safety-extraction}
    \end{subfigure}\quad 
     \begin{subfigure}[t]{0.48\textwidth}
         \centering
        \includegraphics[width=0.95\textwidth]{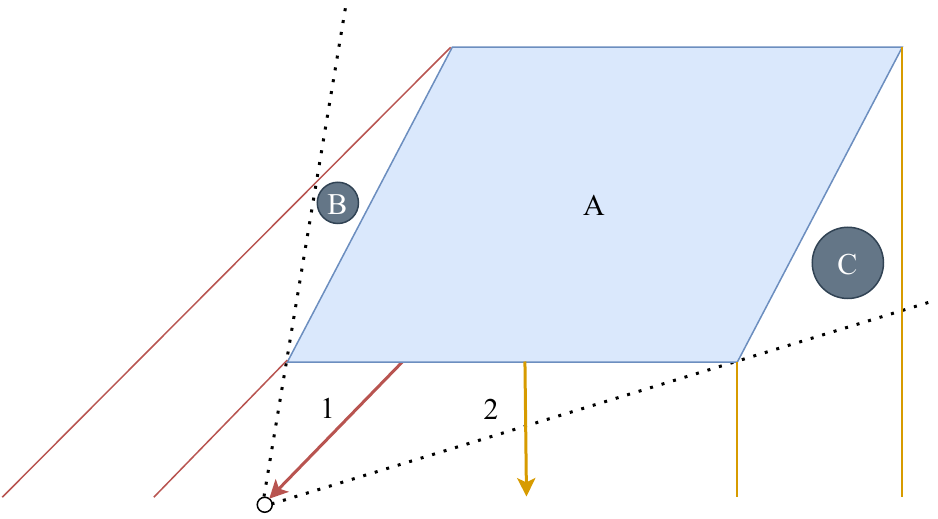}
         \caption{Both direction 1 and 2 do not satisfy the first condition. Without any restriction on the shape of the object, it is nontrivial to find a valid direction to move the object which guarantees collision-free sweeping volume.}
        \label{fig:safety-extraction-2}
     \end{subfigure}\quad 
     \begin{subfigure}[t]{0.48\textwidth}
         \centering
        \includegraphics[width=0.8\textwidth]{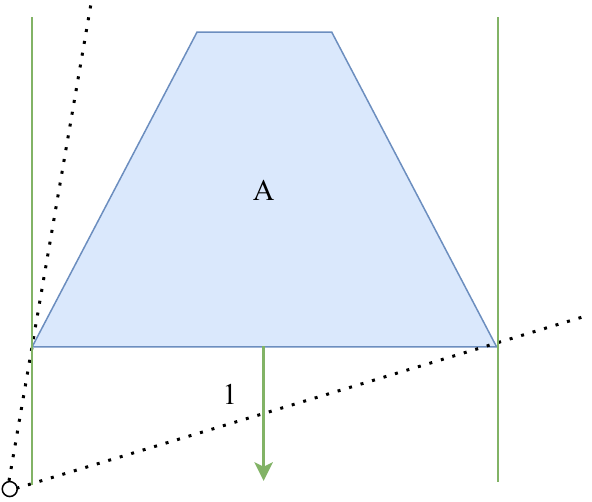}
         \caption{As object $A$ is contained in the sweeping volume of the visible surface. Moving in the perpendicular direction guarantees safety.}
        \label{fig:safety-object-shape-pose}
     \end{subfigure}\quad
     \begin{subfigure}[t]{0.48\textwidth}
         \centering
        \includegraphics[width=0.8\textwidth]{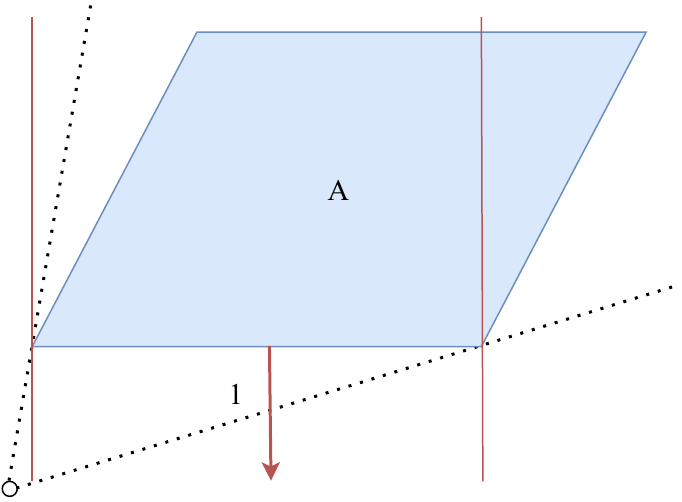}
         \caption{The sweeping volume of perpendicular direction may not always contain the object. The object shape has to be restricted.}
        \label{fig:safety-object-shape-pose-2}
     \end{subfigure}     
    \caption{Illustrations of the Condition on Sweeping Volume of Visible Surface.}
    \label{fig:safety-condition}
\end{figure}

\begin{figure}[ht]
    \centering
     \begin{subfigure}[t]{0.48\textwidth}
         \centering
        \includegraphics[width=0.9\textwidth]{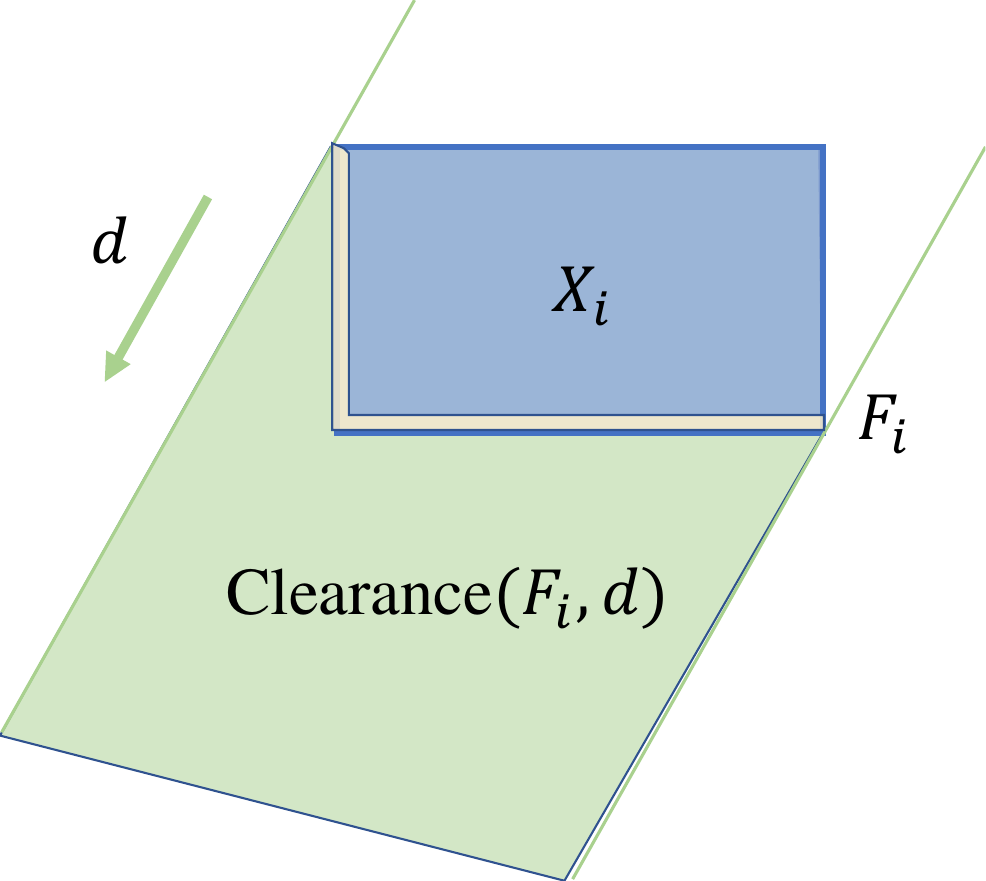}
         \caption{When the object is contained in the sweeping volume of the visible surface, it also has to be convex.
         Then the clearance along direction $d$ is the union of Clearance$(F_i,d)$ and object model $X_i$.}
     \end{subfigure}\quad
     \begin{subfigure}[t]{0.48\textwidth}
         \centering
        \includegraphics[width=0.9\textwidth]{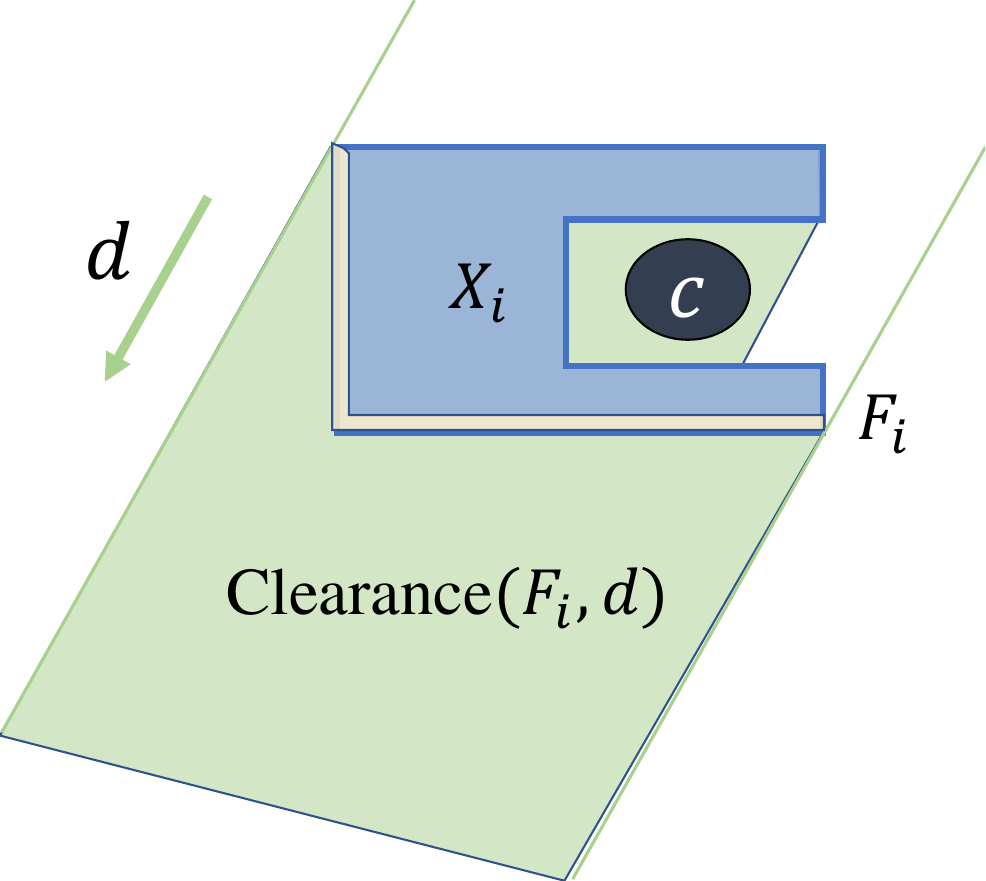}
         \caption{If the object is not convex, the object model may contain "holes" which may lead to collision.}
     \end{subfigure}    
    \caption{Illustrations of Convexity Condition.}
    \label{fig:safety-condition-convex}
\end{figure}

For simplicity, we satisfy the two conditions by restrict the objects to either cylinders or rectangular prisms. Under this assumption, the perpendicular direction to the shelf serves as a safe direction to extract the object.
Note that if the object is a cylinder, under the assumption that the cylinder is far enough from the camera, the perpendicular extraction works as a safe extraction mode.

\end{document}